\documentclass[letterpaper, 10 pt, conference]{ieeeconf}

\IEEEoverridecommandlockouts

\usepackage{amsmath,amssymb}
\usepackage{graphicx}
\usepackage{booktabs}
\usepackage{multirow}
\usepackage{float}
\usepackage{cite}
\usepackage{url}
\usepackage[hidelinks]{hyperref}
\usepackage{bm}

\makeatletter
\let\small\normalsize
\let\footnotesize\normalsize
\let\scriptsize\normalsize
\let\tiny\normalsize
\let\@IEEEabskeysecsize\normalsize
\makeatother

\hypersetup{
  pdftitle={Contact-Constrained Lower-Limb Joint-Offset Calibration for Humanoid Robots},
  pdfauthor={Kaixiang Lu, Haiyu Lan, Chunxiao Qiao, You Li, Chengyuan Luo, Enyu Li, Peiwen Lin, Chuang Wang},
  pdfsubject={Humanoid robot kinematic calibration},
  pdfkeywords={humanoid robots, kinematic calibration, joint encoder offsets, contact constraints, observability}
}

\newtheorem{proposition}{Proposition}

\DeclareMathOperator{\Exp}{Exp}
\DeclareMathOperator{\Log}{Log}

\DeclareMathOperator{\diag}{diag}

\newcommand{\dq}{\delta q}
\newcommand{\SE}{\mathrm{SE}(3)}

\newcommand{\R}{\mathbb{R}}
\newcommand{\rev}[1]{#1}
\newcommand{\revfigure}[1]{#1}
\newcommand{\revcolor}{}

\renewcommand{\QED}{}

\begin{document}
\raggedbottom

\title{\Large \bf
Contact-Constrained Lower-Limb Joint-Offset Calibration\\for Humanoid Robots
}

\author{Kaixiang Lu\textsuperscript{1}, Haiyu Lan\textsuperscript{2,*}, Chunxiao Qiao\textsuperscript{2}, You Li\textsuperscript{1},\\
Chengyuan Luo\textsuperscript{2}, Enyu Li\textsuperscript{2}, Peiwen Lin\textsuperscript{2}, Chuang Wang\textsuperscript{2}\\
\textsuperscript{1}Wuhan University\\
\textsuperscript{2}AGIBOT}

\maketitle
\renewcommand{\thefootnote}{\fnsymbol{footnote}}
\footnotetext[1]{Corresponding author: Haiyu Lan.}
\footnotetext[0]{This paper has been accepted for publication in \emph{IEEE Robotics and Automation Letters} (RA-L).}
\thispagestyle{empty}
\pagestyle{empty}

\begin{abstract}
Accurate joint encoder offsets are essential for kinematic consistency in humanoid lower limbs, yet existing calibration methods typically require external motion-capture systems or fiducial targets.
We present a self-contained calibration framework exploiting only onboard joint encoders and a pelvis-mounted IMU during static double-support contact.
The inter-foot transform from forward kinematics must stay constant when both feet are fixed; minimizing its posture-dependent dispersion yields a nonlinear least-squares problem over the 12-dimensional offset vector.
A Hessian eigenstructure analysis shows that parallel pitch axes induce a rotational coupling.
Orientation residuals then observe only the pitch-offset sum, while translation and posture diversity set the remaining numerical observability.
\rev{For the A3 pitch$\to$roll$\to$yaw ordering, hip-roll and hip-yaw excitation reduce hip-pitch coupling. A standing-posture knee prior then anchors the remaining weak pitch-chain decomposition.}
Simulation and real-machine injection tests show consistent recovery, and on held-out recordings calibration reduces foot-height RMS residuals from $4.26$ to $2.20$\,mm on A3 and from $8.03$ to $1.43$\,mm on A2.
\rev{An independent LiDAR-inertial reference checks the pitch-coupled channel. Removing an injected pitch offset moves the leg-odometry vertical drift back toward the LiDAR trajectory. A few static double-support stances thus provide contact-consistent corrections for well-excited directions. Individual offsets in the weak pitch chain remain prior-dependent.}
\end{abstract}

\section{INTRODUCTION}

Humanoid robots depend on accurate kinematic models for locomotion, whole-body control, and contact planning~\cite{Kuindersma2016Atlas,Hutter2016ANYmal}.
Joint encoder offsets, the persistent biases between encoder and true mechanical zero, accumulate along serial chains and cause foot-placement errors and degraded balance~\cite{Hollerbach1996CalibrationIndex}.
Even sub-degree per-joint offsets can generate centimeter-level distal errors at the foot sole.

Classical calibration identifies kinematic parameters by minimizing discrepancies between model predictions and external measurements from laser trackers, motion-capture systems, or fiducial targets~\cite{Mooring1991ManipulatorCalibration,Hollerbach1996CalibrationIndex}.
Such setups are impractical for repeated in-field recalibration.

A vision-based alternative observes onboard fiducial markers with onboard cameras~\cite{Birbach2012WholeBodySelfCalibration,Maier2015WholeBody}.
For the legs, however, the feet rarely enter the head-camera field of view during normal stances, making foot markers unreliable.
This motivates a calibration signal that needs no external view of the limb.

\begin{figure}[!t]
\centering
\includegraphics[width=0.70\columnwidth]{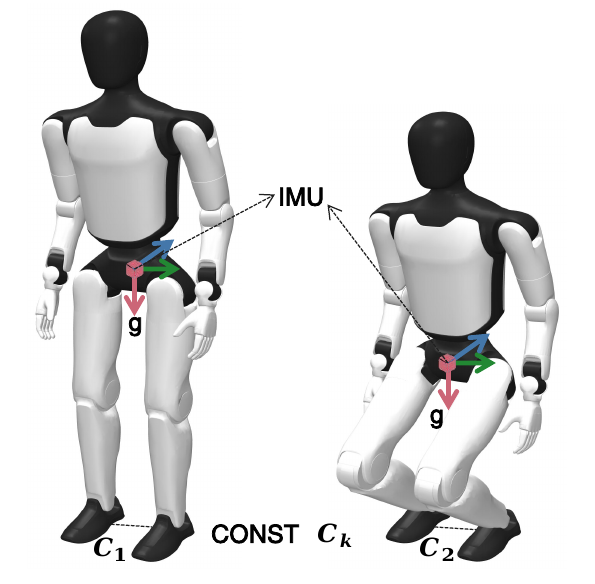}
\caption{Calibration principle. Upright and squat postures share the same fixed-foot transform $C_k$; the pelvis IMU provides gravity $g$.}
\label{fig:concept_overview}
\end{figure}

\begin{figure*}[!t]
\centering
\includegraphics[width=0.94\textwidth]{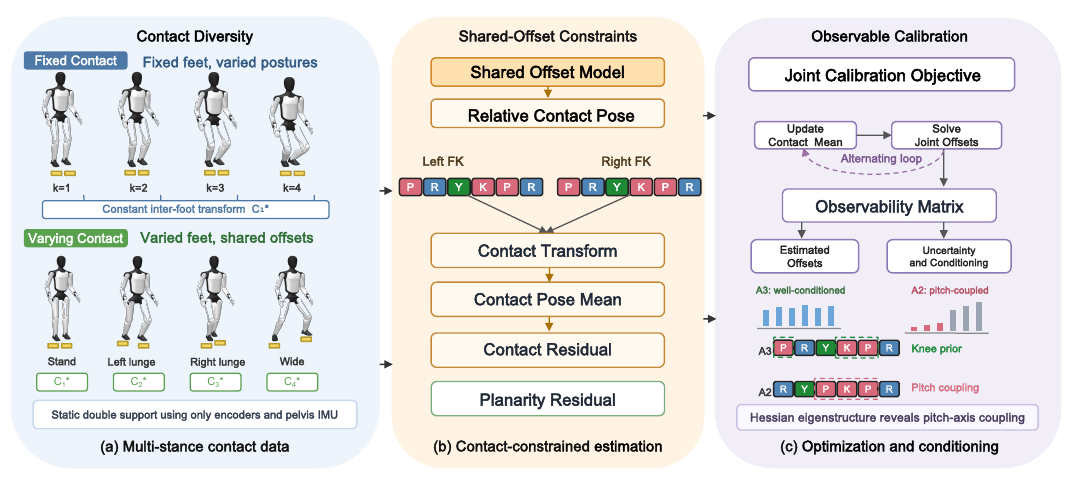}
\caption{Calibration pipeline: collect multi-stance double-support data, form contact and flat-ground residuals, and alternate contact-mean updates with offset optimization. \rev{P/R/Y/K denote pitch/roll/yaw/knee; red, blue, and green boxes encode pitch, roll, and yaw, and dashed outlines mark pitch coupling.}}
\label{fig:flowchart}
\end{figure*}

The key physical constraint available to legged robots is \emph{contact}: during stable double support, both feet remain fixed relative to the ground, and the inter-foot transform must be constant across all postures.
If joint offsets are incorrect, this transform varies with posture, and minimizing its dispersion yields offset estimates.
This idea has appeared implicitly in legged state estimation~\cite{Bloesch2012LeggedStateEstimation,Hartley2020ContactInEKF,Rotella2014LegCalibration}, and Liu and Tang~\cite{Liu2017DoubleSupport} applied double-support constraints to offline humanoid calibration on the NAO platform.
What remains less explicit is how the resulting batch optimization is conditioned by lower-limb joint ordering, especially under the parallel pitch axes common in humanoid legs.

\rev{The AGIBOT A2 uses a roll--yaw--pitch proximal ordering. The AGIBOT A3 uses a pitch--roll--yaw ordering.}

Yamane~\cite{Yamane2011PracticalCalibration} demonstrated practical kinematic calibration for humanoid robots using joint encoders, an IMU, and flat-foot contact constraints.
However, the inherent coupling among the parallel pitch axes of hip-pitch, knee, and ankle-pitch was not formally characterized.
Because consecutive pitch joints rotate about aligned $Y$-axes, their offsets produce collinear rotational Jacobian columns, so only their \emph{sum} is observable from orientation.
We provide a formal Hessian-eigenvector characterization of this coupling and develop decoupling strategies tailored to different kinematic orderings.

The contributions are:
\begin{enumerate}
\item A self-contained lower-limb calibration framework combining inter-foot contact consistency, IMU-aided flat-ground constraints, and Tikhonov regularization with mechanical priors, requiring only onboard joint encoders and a pelvis IMU.
\item \rev{A formal Hessian-based analysis of pitch-axis coupling. The analysis separates information in the data from curvature added by regularization and mechanical priors.}
\item \rev{Validation on A3 simulation and on both real platforms. A same-geometry MuJoCo study isolates the effect of joint ordering. External LiDAR and registered motion-capture tests check the pitch--vertical channel and inter-foot orientation, and a free-standing MuJoCo deployment tests closed-loop locomotion under persistent offsets.}
\end{enumerate}

\section{RELATED WORK}

\subsection{Robot Kinematic Calibration}
Classical manipulator calibration estimates geometric parameters by comparing measured and predicted end-effector poses under Denavit--Hartenberg~\cite{Craig2005Robotics,Mooring1991ManipulatorCalibration} or product-of-exponentials~\cite{Okamura1996POECalibration} parameterizations.
Closed-chain formulations exploit loop constraints for self-calibration without absolute pose measurements~\cite{Wampler1995ImplicitLoop,Khalil2002SelfCalibration,Bennett1991ClosedChainCalib}.
These works address fixed-base serial chains.
Humanoid legs add a floating-base, contact-dependent, parallel-pitch structure with distinct observability challenges~\cite{Khalil2004ModelingIdentification}.

\subsection{Contact-Based Constraints for Legged Robots}
Contact assumptions are central to legged state estimation.
Bloesch et al.~\cite{Bloesch2012LeggedStateEstimation} fuse IMU, encoder, and foot-contact information for real-time pose estimation; Hartley et al.~\cite{Hartley2020ContactInEKF} formulate contact-aided invariant EKF; Camurri et al.~\cite{Camurri2020Pronto} extend this to multi-sensor estimators; Rotella et al.~\cite{Rotella2014LegCalibration} incorporate leg kinematics into humanoid state estimation.
These works use contact constraints for \emph{online state estimation}, not offline calibration.
Our work applies the same physical assumption, foot fixity during double support, to a fundamentally different purpose: estimating persistent joint encoder offsets by dispersion minimization over batch data.

\subsection{Humanoid Calibration}
Yamane~\cite{Yamane2011PracticalCalibration} demonstrated practical kinematic calibration for force-controlled humanoid robots using only joint encoders and an IMU under flat-foot contact; the same work also performed dynamic parameter identification using torque sensors.
Ogawa et al.~\cite{Ogawa2014DynamicsIdentification} added dynamic parameter identification from joint torque sensors and contact forces.
Birbach et al.~\cite{Birbach2012WholeBodySelfCalibration} and Pradeep et al.~\cite{Pradeep2014HumanoidSelfCalibration} self-calibrate the upper body and multi-sensor arms with onboard cameras.
Our work differs from Yamane in three respects.
\rev{First, we formulate the constraint on $\SE$ through inter-foot transform dispersion. We also compare it with Yamane's sole-height method under matched simulated data. The comparison does not show a universal winner.}
Second, we give a formal Hessian-eigenvector analysis of the pitch-axis coupling.
Third, we show how kinematic ordering shapes this coupling, an aspect not analyzed in prior contact-based humanoid calibration.

\subsection{Excitation Design and Observability}
Ayusawa et al.~\cite{Ayusawa2017ExcitationTrajectory} optimize excitation trajectories for inertial parameter identification based on condition-number criteria.
Venture et al.~\cite{Venture2009ExcitationBiped} numerically design motions for biped dynamics identification.
Maier et al.~\cite{Maier2015WholeBody} automatically select calibration configurations by maximizing SVD-based observability indices of the identification Jacobian, achieving accurate whole-body calibration with few measurements.
We extend observability analysis beyond condition numbers and singular-value rankings to explicit eigenvector characterization of the Hessian near-nullspace.
This reveals the geometric origin of weak directions and ties them to the kinematic chain.

\begin{figure*}[!t]
\centering
\includegraphics[width=0.90\textwidth]{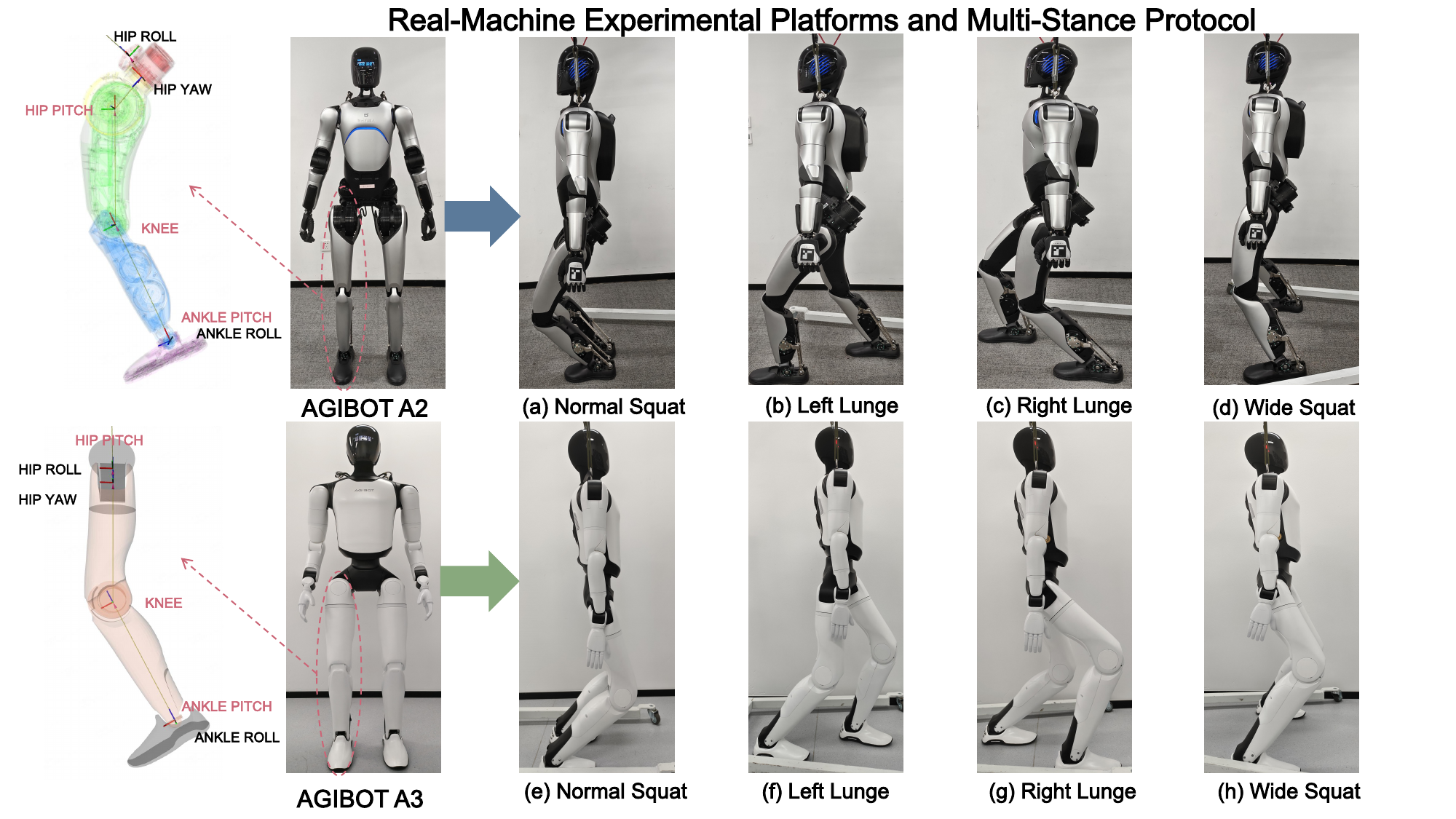}
\caption{\rev{Real-machine A2 and A3 platforms and the four-stance protocol. Only onboard encoders and a pelvis IMU are used.}}
\label{fig:real_photos}
\end{figure*}

\section{METHOD}

\subsection{Kinematic Model and Offset Parameterization}
Consider a humanoid with a 6-DoF lower-limb serial chain per leg.
The two platforms studied are the AGIBOT A2 and A3 humanoid robots, which differ in joint ordering: A3 uses pitch$\to$roll$\to$yaw$\to$knee$\to$ankle-pitch$\to$ankle-roll; A2 uses roll$\to$yaw$\to$pitch$\to$knee$\to$ankle-pitch$\to$ankle-roll.

Let $q_k^{\mathrm{enc}} \in \R^{12}$ denote the encoder readings of both legs at time step $k$, and $\dq \in \R^{12}$ the constant offset vector.
The corrected joint angles are $q_k = q_k^{\mathrm{enc}} + \dq$.
\rev{The forward kinematics from the base to the left and right feet are $T_{BF_L}(q_k)$ and $T_{BF_R}(q_k)$. They are computed from the robot's URDF model.}

\subsection{Inter-Foot Consistency Constraint}
\rev{We use active transforms. $T_{AB}$ maps coordinates from frame $B$ to frame $A$.}
During stable double support, the inter-foot transform:
\begin{equation}
C_k(\dq) = T_{BF_L}(q_k)^{-1} \, T_{BF_R}(q_k)
\label{eq:interfoot}
\end{equation}
must be constant within a contact segment.
\rev{It is the pose of the right foot in the left-foot frame.}
\rev{For a segment of $N$ retained frames, let $C^*$ be the intrinsic mean inter-foot transform on $\SE$~\cite{Pennec2006IntrinsicStatistics,Sola2018MicroLie}:}
\begin{equation}
\revcolor C^*=\arg\min_{C\in\SE}\sum_{k=1}^{N}\left\|\Log\!\left(C^{-1}C_k(\dq)\right)\right\|^2,
\label{eq:frechet_mean}
\end{equation}
\rev{The implementation initializes $C^*$ with $\Exp(\frac{1}{N}\sum_k\Log C_k)$ and refines it by iterative Log--Exp updates after each outer solve.}
The contact-consistency residual is:
\begin{equation}
r_k^{\mathrm{feet}}(\dq) = \Log\!\left((C^*)^{-1} C_k(\dq)\right) \in \R^6.
\label{eq:residual_feet}
\end{equation}
\rev{This residual measures deviation from the segment mean. For multiple segments, $S$ is their count, $\mathcal C_s$ the frame set, and $N_s=|\mathcal C_s|$. Frames are time samples; retained joint-angle cluster representatives are distinct configurations.}
\rev{\emph{Ordering:} $\Log(T)=[\rho_x,\rho_y,\rho_z,\phi_x,\phi_y,\phi_z]^\top$ (translation, then rotation). $J_{\mathrm{data}}$ stacks weighted feet-then-flat rows by segment and time; columns are left/right offsets in the platform order above.}

\subsection{IMU-Aided Flat-Ground Constraint}
A pelvis-mounted IMU provides the gravity direction $g_{\mathrm{imu}}^{(k)}$ in base frame~\cite{Mahony2008NonlinearComplementary}.
\rev{The world frame $W$ uses the opposite gravity direction as its vertical axis. Its yaw is arbitrary.}
Transforming foot positions and normals to the world frame yields a flat-ground residual:
\begin{equation}
r_k^{\mathrm{flat}} = \begin{bmatrix} p_{L,z}^W - p_{R,z}^W \\ z_{L,x}^W \\ z_{L,y}^W \\ z_{R,x}^W \\ z_{R,y}^W \end{bmatrix} \in \R^5,
\label{eq:residual_flat}
\end{equation}
where the first row enforces equal foot heights, and rows 2--5 penalize deviations of each foot normal from the gravity direction.
\rev{A common yaw rotation leaves these terms unchanged. IMU tilt error changes the inferred vertical direction and can bias the result. Uneven or non-coplanar contact can also cause bias.}

\subsection{Optimization Objective}
The calibration objective combines contact consistency, flat-ground constraints, and Tikhonov regularization~\cite{Hansen1998RankDeficient}:
\begin{equation}
\min_{\dq} \sum_{k=1}^{N} \!\left[\lambda_f \|r_k^{\mathrm{feet}}\|^2 + \lambda_g \|r_k^{\mathrm{flat}}\|^2 \right] + \lambda_q \|W_q \!\odot\! \dq\|^2,
\label{eq:objective}
\end{equation}
where $\lambda_f$, $\lambda_g$, $\lambda_q$ are scalar weights, $W_q \in \R^{12}$ is a per-joint weight vector, and $\odot$ denotes element-wise multiplication.
For multiple stances sharing a single $\dq$:
\begin{equation}
\min_{\dq} \sum_{s=1}^{S} \sum_{k \in \mathcal{C}_s} \!\left[\lambda_f \|r_{s,k}^{\mathrm{feet}}\|^2 + \lambda_g \|r_{s,k}^{\mathrm{flat}}\|^2 \right] + \lambda_q \|W_q \!\odot\! \dq\|^2.
\label{eq:multi_bag}
\end{equation}
\rev{Under the platform-specific operating settings, A3 additionally includes a quadratic knee-prior term $\lambda_{\mathrm{knee}}\|\dq_{\mathrm{knee}}-\dq_{\mathrm{knee}}^{\mathrm{prior}}\|^2$ based on~\eqref{eq:knee_prior}, with $\lambda_{\mathrm{knee}}=100$; A2 uses $\lambda_{\mathrm{knee}}=0$. The matched comparisons below use no mechanical prior.}
\rev{The numerical weights for the platform-specific operating settings are:
A3: $(\lambda_f, \lambda_g, \lambda_q, \lambda_{\mathrm{knee}}) = (2, 10, 0.01, 100)$, $W_q = [1,1,1,1.5,1.5,1.5]$ per leg;
A2: $(\lambda_f, \lambda_g, \lambda_q, \lambda_{\mathrm{knee}}) = (1, 100, 1, 0)$, $W_q = [0.5,0.5,0.5,1,1.5,1.5]$ per leg.}
Here $W_q$ follows the platform-specific joint order defined above.
The two-order-of-magnitude difference in $\lambda_q$ reflects the \rev{stronger ridge regularization} used on A2 to bound its pitch-axis near-nullspace.
\rev{The real-data solver retains metres and radians in the residual vector. The scalar weights therefore affect the numerical condition number.}

\subsection{Alternating Optimization}
Since $C_s^*$ depends on $\dq$ through~\eqref{eq:frechet_mean}, we adopt an alternating scheme (Fig.~\ref{fig:flowchart}).
At each outer iteration, $C_s^*$ is recomputed from the current $\dq$ estimate; Ceres Solver~\cite{Agarwal2022Ceres} then updates $\dq$ with $C_s^*$ held fixed using DENSE\_QR~\cite{Nocedal2006NumericalOptimization}.
Convergence occurs within 3--4 outer iterations for all experiments reported.
\rev{We initialize $\dq$ at zero and bound each component to $\pm0.08$\,rad. The inner solve uses at most 200 iterations and a function tolerance of $10^{-8}$. The outer loop uses at most 15 iterations.}

\subsection{Data Preprocessing}
A contact-consistency filter rejects frames violating the fixed-contact assumption.
For each frame $k$, a local Fr\'echet mean $C_k^{\mathrm{local}}$ is computed over a temporal window of $T_w = 2.5$\,s; frames with rotation residual $s_k > 5^\circ$ are discarded.
\rev{The $0.5$\,rad/s gyroscope threshold is a permissive rejection guard. Raw-encoder frames are joint-angle clustered and subsampled to $N_{\max}=500$ per stance; the first solve then removes the largest 10\%, leaving 450 distinct configurations. Filtering is not repeated after correction.}

\section{OBSERVABILITY ANALYSIS AND DECOUPLING}

We distinguish \emph{structural observability}, determined solely by the kinematic topology, from \emph{numerical observability}, which depends on posture diversity and data volume.

\subsection{Hessian Eigenstructure}
\rev{Let $J_{\mathrm{data}}$ be the Jacobian of the contact and flat-ground residuals. We distinguish the data information matrix from the regularized Hessian:}
\begin{equation}
\begin{aligned}
\revcolor H_{\mathrm{data}}&=J_{\mathrm{data}}^\top J_{\mathrm{data}},\\
\revcolor H_{\mathrm{reg}}&=H_{\mathrm{data}}+\lambda_q\diag(W_q^2)+H_{\mathrm{mech}}.
\end{aligned}
\label{eq:hessian}
\end{equation}
\rev{$H_{\mathrm{mech}}$ is the curvature from mechanical priors. It is zero when no such prior is used. Regularization changes $H_{\mathrm{reg}}$ but does not add information to $J_{\mathrm{data}}$. We report $\kappa(J_{\mathrm{data}})$, $\kappa(H_{\mathrm{data}})$, and $\kappa(H_{\mathrm{reg}})$ as numerical conditioning measures. They are not posterior uncertainties.}
\rev{For $N$ retained frames, $r_{\mathrm{data}}\in\R^{11N}$ and $J_{\mathrm{data}}\in\R^{11N\times12}$. The matched four-stance Jacobians have numerical rank 12.}
Eigenvectors corresponding to small eigenvalues identify offset combinations that cannot be reliably estimated.

\subsection{Inherent Pitch-Axis Coupling}

\begin{proposition}[Rotational pitch-axis coupling]
\label{prop:pitch_coupling}
Consider a serial-chain segment containing $n \geq 2$ consecutive revolute joints whose axes remain parallel when expressed in the same residual frame, with common unit direction $a$.
For the rotational component of the inter-foot residual, the corresponding angular-Jacobian columns are collinear.
The orientation residual therefore observes only one linear combination, proportional to $\sum_i \dq_i$; the individual offsets span an $(n{-}1)$-dimensional nullspace of the rotational Jacobian.
\end{proposition}

\begin{proof}
Let $J_\omega$ denote the angular Jacobian of the rotational residual with all joint axes expressed in the same residual frame.
For the consecutive parallel-axis segment, each joint contributes $J_{\omega,i}=a$ up to the same sign convention, so the first-order orientation perturbation is
\begin{equation}
\delta\phi = J_\omega \dq = a \sum_{i=1}^{n} \dq_i .
\end{equation}
Thus every vector $v$ satisfying $\sum_i v_i=0$ lies in the orientation-only nullspace.
\rev{For the hip-pitch, knee, and ankle-pitch triplet, one convenient basis is:}
\begin{equation}
v_1 \propto e_{\mathrm{kn}} - e_{\mathrm{ap}}, \quad v_2 \propto e_{\mathrm{hp}} - e_{\mathrm{kn}}.
\label{eq:nullspace_basis}
\end{equation}
\end{proof}

\textbf{Remark: translational terms and the full SE(3) residual.}
The translational Jacobian columns $\hat{y} \times (p_{\mathrm{foot}} - p_i)$ differ in magnitude due to distinct lever arms, theoretically breaking the rotational degeneracy.
\rev{Compact humanoid geometry limits lever-arm diversity. The recorded stances are also dominated by sagittal-plane motion. The positional information can therefore be weak relative to measurement noise. This produces a \emph{near}-nullspace rather than an exact one in the full SE(3) Hessian. Millimeter-scale lateral offsets for wiring routing and structure further perturb the ideal coupling. Fig.~\ref{fig:hessian_spectrum} quantifies this effect. On A2, the weakest eigenvalues remain separated from the dominant modes. On A3, the spectrum is more compact. The real platforms also differ in geometry, excitation, and sensor placement. The physical comparison alone cannot isolate joint ordering.}

\subsection{Effect of Kinematic Ordering}

\rev{The coupling structure can depend on the kinematic ordering.}

\rev{\textbf{A3, pitch$\to$roll$\to$yaw:} Hip-pitch is the first joint and is attached to the base. When hip roll and hip yaw are nonzero, the hip-pitch axis is no longer parallel to the knee and ankle-pitch axes in the foot frame.}

\rev{\textbf{A2, roll$\to$yaw$\to$pitch:} Hip-pitch, knee, and ankle-pitch are consecutive joints. There is no intervening nonparallel axis. Translational lever arms provide only weak additional information in the recorded stances.}

\subsection{Decoupling and Resolution Strategies}

We distinguish two mechanisms: \emph{geometric decoupling}, which breaks axis parallelism through posture diversity, and \emph{algebraic resolution}, which injects external prior information to constrain the remaining nullspace.

\subsubsection{\texorpdfstring{\rev{Geometric Decoupling via Roll and Yaw Excitation}}{Geometric Decoupling via Roll and Yaw Excitation}}
\rev{Applying nonzero hip-roll and hip-yaw during data collection rotates the hip-pitch axis out of alignment with knee and ankle-pitch. The flat-ground constraint~\eqref{eq:residual_flat} uses the IMU-derived orientation. It provides independent information about hip-pitch. This reduces the nullspace from 2D to the knee--ankle-pitch coupling.}

\subsubsection{Algebraic Resolution via Knee Prior}
A brief straight-leg standing phase of roughly $5$\,s provides a knee prior:
\begin{equation}
\dq_{\mathrm{knee}} \approx q_{\mathrm{knee}}^{\mathrm{mech}} - \bar{q}_{\mathrm{knee}}^{\mathrm{enc}},
\label{eq:knee_prior}
\end{equation}
incorporated as a quadratic penalty with weight $\lambda_{\mathrm{knee}}$.
\rev{Once the knee is anchored, the pitch-sum constraint determines the remaining coupled pitch component. The knee prior does not make the individual offsets observable. It only selects one decomposition in the coupled subspace.}

\rev{The mechanical zero $q_{\mathrm{knee}}^{\mathrm{mech}}$ is obtained from the URDF and CAD model. Any model error propagates into the selected individual offsets.}
The method assumes the initial offset is already small enough for the robot to reach a near-straight-leg posture; this matches commercial humanoid practice, where factory residual offsets are typically small.

\subsubsection{Strengthened Regularization for A2}
On A2, three consecutive parallel-axis joints preclude geometric decoupling.
\rev{Tikhonov regularization with joint weights suppresses near-nullspace drift. It can bias any individual estimate. Its weights and centers must therefore be reported and tested.}

\rev{For each leg, the orientation residual directly constrains the pitch sum $s=\delta q_{\mathrm{hp}}+\delta q_{\mathrm{knee}}+\delta q_{\mathrm{ap}}$. We report $s_L$, $s_R$, and $s_R-s_L$ separately from the individual offsets.}

\begin{figure}[t]
\centering
\revfigure{\includegraphics[width=0.96\columnwidth]{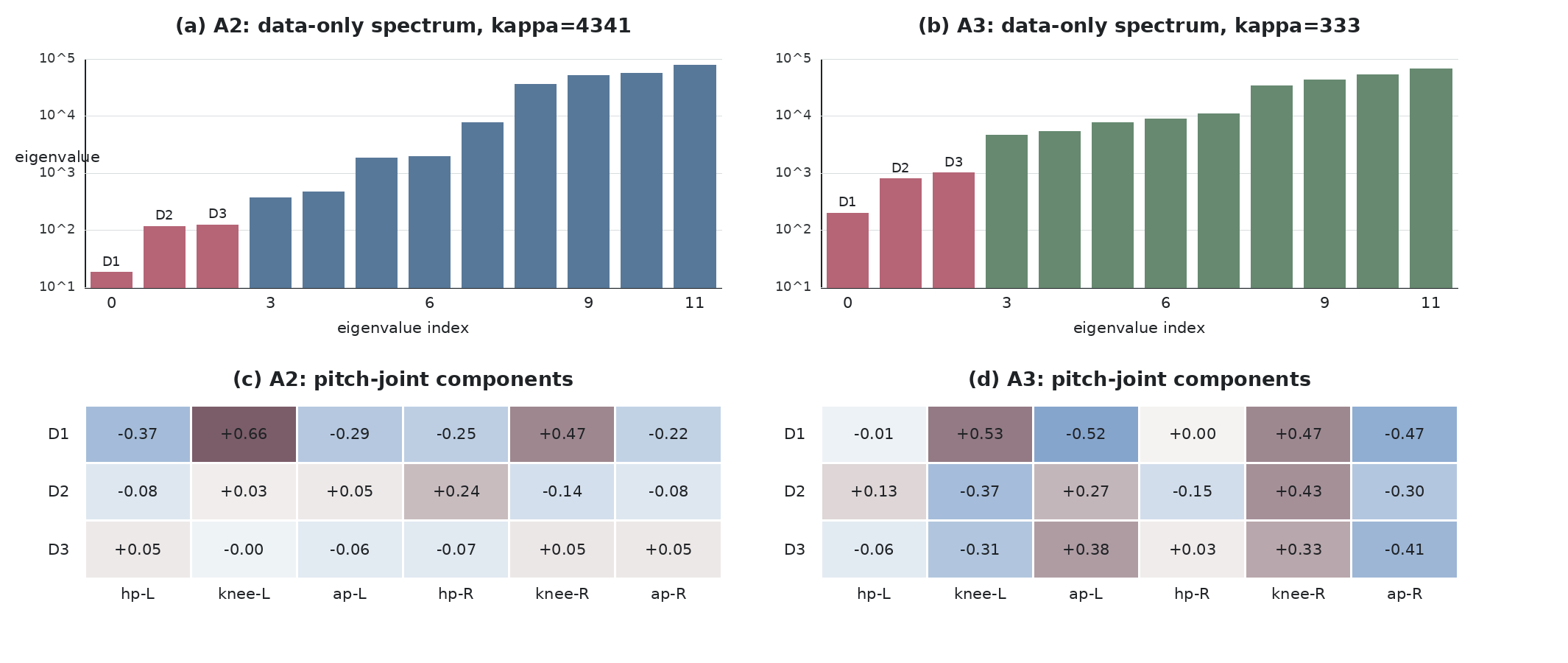}}
\caption{\rev{Matched four-stance real data. Top: $H_{\mathrm{data}}$ eigenvalues, the squared singular values of $J_{\mathrm{data}}$. Bottom: pitch components of its three weakest eigenvectors, D1--D3. The physical-platform gap is not ordering-only.}}
\label{fig:hessian_spectrum}
\end{figure}

\section{EXPERIMENTS}

\subsection{Experimental Setup}

\textbf{A3 Simulation.}
The AGIBOT A3 humanoid with 12 lower-limb DoF is simulated in MuJoCo~\cite{MuJoCo} with weld constraints fixing both feet.
\rev{The A3 simulation uses one natural double-support squat MuJoCo bag with continuous posture variation and retains 450 frames.}
\rev{This single stance has $\kappa(H_{\mathrm{reg}})\approx1.1\times10^4$.}
A known multi-degree offset pattern is injected symmetrically across joint pairs.
This injection is intentionally large enough to rise above sensor noise and session variability, yet small enough to remain in the feasible static-contact regime.
It is therefore a conservative but physically plausible recovery test.
The contact, gravity, offset, and knee-prior weights are selected on a separate pilot tuning set and then kept fixed for each platform.

\textbf{A2 Real Machine.}
The AGIBOT A2 humanoid uses four recordings: Stance~I a normal squat, II a left lunge, III a right lunge, and IV a wide squat, with several hundred static frames retained per stance.
\rev{The feet remain fixed while the posture changes continuously. Each stance contributes 450 retained frames.}
Four stances are used as a conservative operating point: testing showed that the first few diverse postures close the dominant observability gaps, while additional poses mainly increase collection time with little numerical gain.
Stronger offset regularization is used on A2 to suppress its pitch-axis near-nullspace.
The real-machine platforms and stance protocol are shown in Fig.~\ref{fig:real_photos}.

\textbf{A3 Real Machine.}
The A3 real-machine experiments use the same four-stance protocol.
The same weight-selection procedure is used, with milder offset regularization because the A3 ordering gives better geometric decoupling.
\rev{The platform-specific A3 operating run and the repeatability sessions use a knee prior from an initial straight-leg phase; the matched comparison is prior-free.}

\rev{\textbf{Controlled MuJoCo comparison.}
We also build two models with the same geometry, masses, feet, IMU, actuators, gains, targets, and noise.
Both models use four stances, scaling translation by 1\,mm and rotation/normal by 1\,mrad.
Only the proximal R-Y-P or P-R-Y ordering changes.
This experiment isolates the ordering effect.}

\textbf{Metrics.}
\rev{The signed per-joint injection error is $e_i=\hat{\dq}_i-\dq_{\mathrm{inject},i}$ and its magnitude is $\epsilon_i=|e_i|$. For real recordings, we use the paired error $\Delta_i$. It subtracts the no-injection estimate from the injected estimate before removing the known injection. This metric tests estimator consistency. It does not measure the native mechanical zero. We report data-only and regularized condition numbers separately.}

\subsection{A3 Simulation Results}

Table~\ref{tab:a3_sim_inject} presents the injection-recovery results.
\rev{The recovery error peaks at $0.27^\circ$ on hip-pitch. It concentrates along the weakly-observed pitch directions of Proposition~\ref{prop:pitch_coupling}. The contact residual constrains the pitch sum, while the anchored knee fixes the selected decomposition within the weak pitch chain. This test shows response to an injected offset under the declared objective. It does not prove that every individual offset is data-observable.}

\rev{In matched A2-like/A3-like MuJoCo models, ideal zero-noise recovery is below $1.2\times10^{-13}$ degrees; weld and settling give RMS floors of $0.092^\circ$ and $0.121^\circ$, respectively.}

\begin{table}[!tbp]
\centering
\caption{\rev{A3 simulation injection recovery for one double-support squat MuJoCo bag with continuous posture variation. The table reports signed error $e_i=\hat{\dq}_i-\dq_{\mathrm{inject},i}$; RMS and maximum rows use $|e_i|$.}}
\label{tab:a3_sim_inject}
\footnotesize
\setlength{\tabcolsep}{2.5pt}
\revcolor
\begin{tabular}{@{}lrrr@{}}
\toprule
Joint & Injected [$^\circ$] & Estimated [$^\circ$] & Error [$^\circ$] \\
\midrule
L hip pitch   & $+1.719$ & $+1.993$ & $+0.274$ \\
L hip roll    & $+2.292$ & $+2.216$ & $-0.076$ \\
L hip yaw     & $-1.146$ & $-1.094$ & $+0.052$ \\
L knee        & $-2.292$ & $-2.227$ & $+0.065$ \\
L ankle pitch & $+1.146$ & $+1.147$ & $+0.001$ \\
L ankle roll  & $-1.719$ & $-1.671$ & $+0.047$ \\
R hip pitch   & $+1.719$ & $+1.459$ & $-0.260$ \\
R hip roll    & $+2.292$ & $+2.220$ & $-0.072$ \\
R hip yaw     & $-1.146$ & $-1.206$ & $-0.061$ \\
R knee        & $-2.292$ & $-2.231$ & $+0.061$ \\
R ankle pitch & $+1.146$ & $+1.146$ & $+0.000$ \\
R ankle roll  & $-1.719$ & $-1.755$ & $-0.036$ \\
\midrule
\multicolumn{3}{@{}l}{RMS error} & $\mathbf{0.120}$ \\
\multicolumn{3}{@{}l}{Maximum absolute error} & $\mathbf{0.274}$ \\
\multicolumn{3}{@{}l}{$\kappa(H_{\mathrm{reg}})$} & $\mathbf{1.11 \times 10^4}$ \\
\bottomrule
\end{tabular}
\end{table}

\begin{table*}[!t]
\centering
\caption{\rev{Real-machine injection recovery under the platform-specific operating settings. All joint values are in degrees. $\Delta_i$ measures paired injection consistency. It does not measure native mechanical-zero accuracy.}}
\label{tab:a2_vs_a3_real_inject}
\footnotesize
\setlength{\tabcolsep}{3.0pt}
\revcolor
\begin{tabular}{@{}l|rrrr|rrrr@{}}
\toprule
& \multicolumn{4}{c|}{A2 real robot} & \multicolumn{4}{c}{A3 real robot} \\
\cmidrule(lr){2-5} \cmidrule(l){6-9}
Joint & Base & Injected & Estimated & $\Delta$ & Base & Injected & Estimated & $\Delta$ \\
\midrule
\textbf{L hip pitch}   & $\mathbf{+0.430}$ & $\mathbf{+1.719}$ & $\mathbf{+2.132}$ & $\mathbf{-0.016}$ & $\mathbf{+1.035}$ & $\mathbf{+1.719}$ & $\mathbf{+2.904}$ & $\mathbf{+0.150}$ \\
L hip roll    & $-1.127$ & $+2.292$ & $+1.157$ & $-0.008$ & $+0.495$ & $+2.292$ & $+2.802$ & $+0.016$ \\
L hip yaw     & $+0.244$ & $-1.146$ & $-0.894$ & $+0.008$ & $+0.273$ & $-1.146$ & $-0.887$ & $-0.015$ \\
\textbf{L knee}        & $\mathbf{+1.154}$ & $\mathbf{-2.292}$ & $\mathbf{-1.111}$ & $\mathbf{+0.027}$ & $\mathbf{-0.208}$ & $\mathbf{-2.292}$ & $\mathbf{-2.506}$ & $\mathbf{-0.006}$ \\
\textbf{L ankle pitch} & $\mathbf{-0.450}$ & $\mathbf{+1.146}$ & $\mathbf{+0.685}$ & $\mathbf{-0.011}$ & $\mathbf{-0.850}$ & $\mathbf{+1.146}$ & $\mathbf{+0.290}$ & $\mathbf{-0.006}$ \\
L ankle roll  & $+0.020$ & $-1.719$ & $-1.699$ & $+0.000$ & $+0.340$ & $-1.719$ & $-1.347$ & $+0.032$ \\
\textbf{R hip pitch}   & $\mathbf{+0.521}$ & $\mathbf{+1.719}$ & $\mathbf{+2.232}$ & $\mathbf{-0.007}$ & $\mathbf{+1.001}$ & $\mathbf{+1.719}$ & $\mathbf{+2.583}$ & $\mathbf{-0.137}$ \\
R hip roll    & $+1.115$ & $+2.292$ & $+3.402$ & $-0.005$ & $-0.588$ & $+2.292$ & $+1.712$ & $+0.008$ \\
R hip yaw     & $+0.735$ & $-1.146$ & $-0.408$ & $+0.003$ & $-0.194$ & $-1.146$ & $-1.321$ & $+0.019$ \\
\textbf{R knee}        & $\mathbf{+0.380}$ & $\mathbf{-2.292}$ & $\mathbf{-1.893}$ & $\mathbf{+0.019}$ & $\mathbf{-0.090}$ & $\mathbf{-2.292}$ & $\mathbf{-2.388}$ & $\mathbf{-0.006}$ \\
\textbf{R ankle pitch} & $\mathbf{-0.858}$ & $\mathbf{+1.146}$ & $\mathbf{+0.276}$ & $\mathbf{-0.011}$ & $\mathbf{+0.024}$ & $\mathbf{+1.146}$ & $\mathbf{+1.164}$ & $\mathbf{-0.006}$ \\
R ankle roll  & $-1.386$ & $-1.719$ & $-3.103$ & $+0.001$ & $+0.427$ & $-1.719$ & $-1.321$ & $-0.030$ \\
\midrule
RMS $|\Delta|$ [$^\circ$] & \multicolumn{4}{c|}{$\mathbf{0.012}$} & \multicolumn{4}{c}{$\mathbf{0.061}$} \\
$\kappa(H_{\mathrm{reg}})$ & \multicolumn{4}{c|}{$\mathbf{5.03\times10^3}$} & \multicolumn{4}{c}{$\mathbf{7.22\times10^1}$} \\
\bottomrule
\end{tabular}
\end{table*}

\subsection{A2 and A3 Real-Machine Results}

\rev{Table~\ref{tab:a2_vs_a3_real_inject} retains the platform-specific operating settings; its A3 run uses the knee prior with $\lambda_{\mathrm{knee}}=100$. Its value $\kappa(H_{\mathrm{reg}})=72.27$ is therefore prior-assisted.}
\rev{These settings use different residual weights, joint weights, and ridge strengths. The two condition numbers in this table are therefore not a controlled test of joint ordering. The paired error $\Delta_i$ can also cancel a shared regularization bias. A small $\Delta_i$ does not prove that an individual offset is accurate.}

\rev{We repeated the comparison with common weights and no mechanical prior. The data-only condition number $\kappa(H_{\mathrm{data}})$ is $4340.79$ on A2 and $333.23$ on A3. When the segment mean is recomputed during differentiation, the values are $6254.53$ and $771.61$. This is only platform-level evidence. The robots still differ in geometry, excitation, IMU mounting, and model error.}

\rev{The controlled MuJoCo pair isolates the proximal ordering. For R-Y-P, $\kappa(J_{\mathrm{data}})$ is $348.90$ and $\kappa(H_{\mathrm{data}})$ is $121728.69$. For P-R-Y, the values are $90.69$ and $8224.49$. This result supports an ordering effect for the tested geometry and excitation.}

\rev{In the matched data-only A2 run, the six individual pitch errors have an RMS value of $0.366^\circ$, while the errors in $s_L$, $s_R$, and $s_R-s_L$ are $-0.0053^\circ$, $+0.0165^\circ$, and $+0.0218^\circ$. Moving both knee-prior centers from $-1^\circ$ to $+1^\circ$ changes individual estimates by $0.59^\circ$ to $1.90^\circ$, but the three pitch combinations by at most $0.0071^\circ$.}

\subsection{\texorpdfstring{\rev{Comparison with the Closest Method}}{Comparison with the Closest Method}}
\rev{We implemented the kinematic objective in Eqs.~(1)--(5) of Yamane~\cite{Yamane2011PracticalCalibration}. The implementation estimates one root orientation for each frame and one shared offset vector. It uses the IMU orientation term and the variance of eight sole-vertex heights. We use the same simulated states, noise, and data split for both methods. Over ten noise seeds, the proposed method and Yamane's method give similar results on A2. Yamane's method gives a lower individual-offset RMS value on A3. Held-out foot-pose results also depend on the platform. The results do not support a claim of universal superiority. The proposed objective directly adds inter-foot $x$, $y$, and yaw channels.}

\subsection{Non-Injection Held-Out Validation}

To complement injection-recovery, we evaluate flat-ground consistency on held-out bags without artificial offset injection.
Offsets are estimated from a four-stance set and validated on a separate batch: for A3 from an independent session, and for A2 from excluded recordings with measurable baseline mismatch.

Fig.~\ref{fig:heldout_validation} shows per-frame held-out foot-height residuals before and after calibration.
On A3, pooled RMS foot-height residual decreases from $4.26$ to $2.20$\,mm; on A2, from $8.03$ to $1.43$\,mm; sole-normal residuals change in the same direction on both platforms.
This is not external ground truth, but it tests whether the calibrated offsets improve excluded data under the method's own flat-ground assumptions.
\rev{Each platform uses three held-out sessions with 500 frames per session.}

\begin{figure}[!tbp]
\centering
\includegraphics[width=\columnwidth]{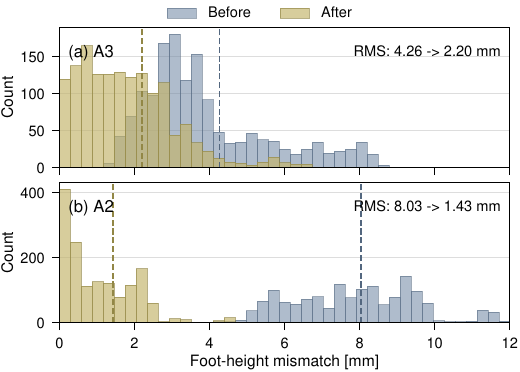}
\caption{Held-out flat-ground validation without injected offsets. Histograms show per-frame foot-height residuals before and after calibration, pooled over validation recordings not used for optimization. \rev{Each platform contributes 1500 held-out frames.} Dashed lines denote RMS values.}
\label{fig:heldout_validation}
\end{figure}

\subsection{External Validation}
\label{sec:lidar_validation}

The tests above share the FK and flat-ground model of the calibration objective.
They measure internal consistency rather than absolute external accuracy.
\rev{We therefore retain the LiDAR-inertial cross-check as an external test of the pitch--vertical channel.}

\rev{Leg kinematic--inertial odometry fuses pelvis-IMU propagation with stance-foot forward kinematics~\cite{Bloesch2012LeggedStateEstimation,Hartley2020ContactInEKF,Rotella2014LegCalibration}. During contact, the stance foot is fixed. The base pose is}
\begin{equation}
\revcolor
T^{W}_{B}(t) = T^{W}_{F}\;T_{BF}\!\big(q^{\mathrm{enc}}(t)+\dq\big)^{-1},
\label{eq:legodom}
\end{equation}
\rev{Here $T^{W}_{F}$ is the fixed stance-foot pose in the world frame. LiDAR-inertial odometry uses FAST-LIO2~\cite{Xu2022FastLIO2}. It uses neither leg encoders nor FK.}

\rev{Large physical offsets could cause a fall. We therefore add known offsets to two recorded walks. A symmetric hip-pitch offset produces vertical drift that is linear in the offset. The sensitivity is $1.04$\,m per degree. Calibration moves the trajectory toward the LiDAR reference in Fig.~\ref{fig:lidar_val}. Table~\ref{tab:lidar_val} reports 16 held-out segments. The vertical RMS error decreases by $25.5\%$.}

\rev{This test checks only the pitch--vertical channel and not absolute zeros. On a separate A3, we register bilateral optical rigid bodies in a flat, parallel-foot reference recording and evaluate a second placement without fitting FK or offsets. The held-out inter-foot orientation error decreases from $5.223^\circ$ to $4.085^\circ$ (21.8\%).}

\begin{figure}[!tbp]
\centering
\includegraphics[width=\columnwidth]{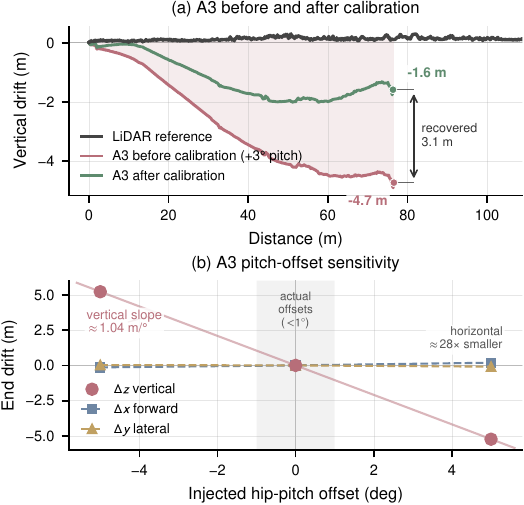}
\caption{\rev{A3 validation against LiDAR-inertial odometry. The left panel shows a held-out walk with a $+3^\circ$ pitch offset. The right panel shows a vertical sensitivity of $1.04$\,m per degree.}}
\label{fig:lidar_val}
\end{figure}

\begin{table}[!tbp]
\centering
\caption{\rev{A3 leg-odometry RMS error against LiDAR before and after correction. Values are the mean and standard deviation over 16 held-out segments.}}
\label{tab:lidar_val}
\footnotesize
\setlength{\tabcolsep}{3.5pt}
\revcolor
\begin{tabular}{@{}lccc@{}}
\toprule
RMS error [mm] & Before & After & Reduction \\
\midrule
Vertical $z$ & $319\pm258$ & $238\pm137$ & $\mathbf{25.5\%}$ \\
Forward $x$                  & $287\pm205$ & $285\pm204$ & $\approx0$ \\
Lateral $y$                  & $287\pm179$ & $288\pm180$ & $\approx0$ \\
Total 3-D                    & $576\pm274$ & $494\pm263$ & $14.2\%$ \\
\bottomrule
\end{tabular}
\end{table}

\subsection{\texorpdfstring{\rev{Closed-Loop Simulation}}{Closed-Loop Simulation}}
\rev{We deploy the same installed A3 walking policy in free-standing MuJoCo before and after applying corrections estimated from independent static data and frozen before locomotion.
Across 35 trials, the policy, gains, model, and commands remain fixed, with paired noise seeds.
At $3^\circ$, foot-strike and pelvis-trajectory RMS deviations fall from $1190.61$ to $26.57$\,mm and from $1224.10$ to $27.16$\,mm.
All trials remain upright; torque and power do not improve consistently. This supports nominal tracking, not expanded stability, energy, or hardware claims.}

\subsection{Stance Diversity and Saturation}

\rev{Table~\ref{tab:ablation} and Fig.~\ref{fig:ablation} show that most A2 conditioning gain comes from the first added diverse stance, with $\kappa(H_{\mathrm{reg}})$ dropping from $1.27\times10^5$ to $6.19\times10^3$.
Further stances saturate, as added contact geometries introduce stance-dependent contact and flatness errors.
This supports a compact four-stance protocol rather than a larger calibration set.}

\begin{table}[!tbp]
\centering
\caption{\rev{Real-A2 stance diversity; I--IV denotes the four stance sets. Eig. 0--2 and $\kappa$ are from $H_{\mathrm{reg}}$.}}
\label{tab:ablation}
\footnotesize
\setlength{\tabcolsep}{1.0pt}
\revcolor
\begin{tabular}{@{}lccccc@{}}
\toprule
Stances & Poses & Eig. 0 & Eig. 1 & Eig. 2 & $\kappa$ \\
\midrule
1 & I only         & 1.1  & 2.4  & 16.3  & $1.27\times10^5$ \\
2 & I + II         & 44.4 & 76.8 & 387.8 & $6.19\times10^3$ \\
3 & I + II + III   & 91.4 & 177.5 & 753.4 & $4.52\times10^3$ \\
4 & I--IV & 109.8 & 558.9 & 1050.6 & $5.03\times10^3$ \\
\bottomrule
\end{tabular}
\end{table}

\begin{figure}[!tbp]
\centering
\includegraphics[width=\columnwidth]{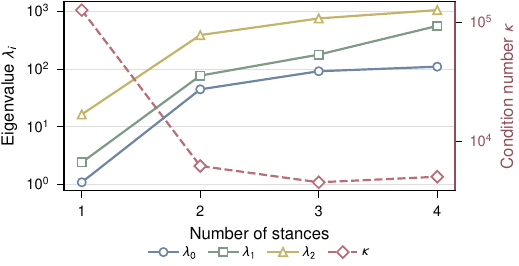}
\caption{\rev{A2 stance diversity for $H_{\mathrm{reg}}$: three smallest eigenvalues and $\kappa(H_{\mathrm{reg}})$.}}
\label{fig:ablation}
\end{figure}

\subsection{Ablation: Regularization, Constraints, and Priors}
\label{sec:ablation_reg}
\rev{On A2, varying $\lambda_q$ barely changes the data fit but moves weak pitch offsets: it selects a stable solution without creating observability. Removing the flat-ground term shifts frontal and transverse hip offsets by $2^\circ$ to $5^\circ$, confirming its out-of-plane information.}
\rev{Under the platform-specific operating setting, the A3 knee prior improves $\kappa(H_{\mathrm{reg}})$ from approximately $333$ to $72.27$. This change is prior curvature in $H_{\mathrm{reg}}$; it does not add information to $H_{\mathrm{data}}$.}

\subsection{A3 Multi-Trial Repeatability}

\begin{table}[H]
\centering
\caption{\rev{A3 real-machine correction repeatability over three rerun sessions. Values are in degrees; they are not mechanical-zero errors.}}
\label{tab:a3_repeatability}
\scriptsize
\setlength{\tabcolsep}{6pt}
\renewcommand{\arraystretch}{0.88}
\revcolor
\begin{tabular}{@{}lrr@{}}
\toprule
Joint & Mean $\dq$ & Standard deviation \\
\midrule
Left hip pitch & $-0.014$ & $0.291$ \\
Right hip pitch & $-0.716$ & $0.597$ \\
Left hip roll & $-0.044$ & $0.165$ \\
Right hip roll & $+0.150$ & $0.418$ \\
Left hip yaw & $+0.193$ & $0.238$ \\
Right hip yaw & $+0.805$ & $1.050$ \\
Left knee & $-0.235$ & $0.114$ \\
Right knee & $-0.258$ & $0.185$ \\
Left ankle pitch & $-0.895$ & $0.104$ \\
Right ankle pitch & $-0.131$ & $0.190$ \\
Left ankle roll & $+0.465$ & $0.329$ \\
Right ankle roll & $-0.038$ & $0.764$ \\
\bottomrule
\end{tabular}
\end{table}

\section{CONCLUSION}

We presented a self-contained lower-limb joint-offset calibration framework exploiting inter-foot contact consistency, IMU-aided flat-ground constraints, and mechanical priors.
\rev{The Hessian analysis shows that parallel pitch axes create weak directions. Matched physical-platform data show better conditioning on A3. A same-geometry simulation supports an ordering effect. A2 pitch sums are more stable than their individual decomposition. Experiments on simulation and real humanoid platforms show that static double-support stances improve contact consistency using only onboard sensing.}

\textbf{Limitations and future work.}
\rev{The method assumes fixed, approximately coplanar double support, and individual pitch offsets can remain prior-dependent. We interpret the estimates as contact-consistent corrections. LiDAR checks pitch--vertical, while registered motion capture independently validates held-out inter-foot orientation; the closed-loop test is A3 simulation. Individual mechanical zeros, translation accuracy, and real-machine locomotion deployment remain future work.}

\bibliographystyle{IEEEtran}
\bibliography{refs}

\end{document}